\documentclass{article}[11pt] 
\usepackage{url}
\usepackage{amsmath}
\usepackage{amssymb}
\usepackage{amsthm}

\usepackage{algorithm}
\usepackage[noend]{algpseudocode}

\usepackage{graphicx}
\usepackage[margin=1in]{geometry}
\usepackage{array}
\newcolumntype{C}[1]{>{\centering}m{#1}}

\date{\nonumber}

\newtheorem{theorem}{Theorem}[section]

\newtheorem{lemma}[theorem]{Lemma}

\newtheorem{definition}[theorem]{Definition}

\newcommand{\ignore}[1]{}

\newcommand{\R}{{\mathbb R}}

\newcommand{\E}{{\mathbb E}}

\newcommand{\eps}{\varepsilon}
\newcommand{\sign}{\operatorname{sign}}

\newcommand{\sig}{\sigma}
\newcommand{\tsig}{\tilde{\sigma}}
\newcommand{\tR}{\tilde{R}}
\newcommand{\eqdef}{:=}

\newcommand{\dist}{p}

\newcommand{\numrank}{\operatorname{sr}}
\newcommand{\numdensity}{\operatorname{nd}}
\newcommand{\numrowdensity}{\operatorname{nrd}}
\newcommand{\nnz}{\operatorname{nnz}}

\title{Near-Optimal Entrywise Sampling for Data Matrices}

\author{
\makebox[0.4\linewidth]{Dimitris Achlioptas}\\
UC Santa Cruz\\
\texttt{optas@cs.ucsc.edu} \\
\and
\makebox[0.4\linewidth]{Zohar Karnin}\\
Yahoo Labs\\
\texttt{zkarnin@ymail.com} \\
\and
\makebox[0.4\linewidth]{Edo Liberty}\\
Yahoo Labs \\
\texttt{edo.liberty@ymail.com} \\
}

\begin{document}

\maketitle


\begin{abstract}

We consider the problem of selecting non-zero entries of a matrix
$A$ in order to produce a sparse sketch of it, $B$, that minimizes $\|A-B\|_2$.
For large $m \times n$ matrices, such that $n \gg m$ (for example, representing
$n$ observations over $m$ attributes) we give sampling distributions that exhibit four
important properties. First, they have closed forms computable from minimal
information regarding $A$. Second, they allow sketching of matrices whose
non-zeros are presented to the algorithm in arbitrary order as a stream, with
$O(1)$ computation per non-zero. Third, the resulting sketch matrices are not
only sparse, but their non-zero entries are highly compressible. 
Lastly, and most importantly, under mild assumptions, our distributions are
provably competitive with the optimal offline distribution. Note that the
probabilities in the optimal offline distribution may be complex functions of
all the entries in the matrix. Therefore, regardless of computational
complexity, the optimal distribution might be impossible to compute in the streaming model.

\end{abstract}

\section{Introduction} 

Given an $m \times n$ matrix $A$, it is often desirable to find a sparser
matrix $B$ that is a good proxy for $A$. Besides being a natural mathematical
question, such sparsification has become a ubiquitous preprocessing step in a
number of data analysis operations including approximate eigenvector
computations \cite{achlioptas2001fast, AroraHazanKale2006,
AchlioptasMcsherry2007}, semi-definite programming~\cite{arora2005fast,
d2008subsampling}, and matrix completion problems \cite{candes2009exact,
candes2010power}.

A fruitful measure for the approximation of $A$ by $B$ is the spectral norm of
$A-B$, where for any matrix $C$ its spectral norm is defined as
$\|C\|_2 	=\max_{\|x\|_2 =1} \|Cx\|_2$. Randomization has been central in
the context of matrix approximations and the overall problem is typically cast
as follows: given a matrix $A$ and a budget $s$, devise a distribution over
matrices $B$ such that the (expected) number of non-zero entries in $B$
is at most $s$ and $\|A-B\|_2$ is as small as possible.

Our work is motivated by big data matrices that are generated by measurement
processes. Each of the $n$ matrix columns correspond to an observation of $m$
attributes. Thus, we expect $n \gg m$. Also we expect the total number of
non-zero entries in $A$ to exceed available memory. We assume that the
original data matrix $A$ is accessed in the streaming model where we
know only very basic features of $A$ a priori and the actual non-zero entries are presented to us one
at a time in an arbitrary order. The streaming model is especially important
for tasks like recommendation engines where user-item preferences become
available one by one in an arbitrary order. But, it is also important in cases
when $A$ exists in durable storage and random access of its entries is
prohibitively expensive.

We establish that for such matrices the following approach gives
provably near-optimal sparsification.  Assign to each element $A_{ij}$ of the matrix a
weight that depends only on the elements in its row $q_{ij} = |A_{ij}|/\|A_{(i)}\|_1$. Take $\rho$ to 
be an (appropriate) distribution over the rows. Sample $s$ i.i.d.\ entries from $A$ using the distribution
$\dist_{ij} = \rho_i q_{ij}$. Return $B$ which is the mean of $s$ matrices, each containing a single non zero entry $A_{ij}/p_{ij}$ in the selected location $(i,j)$.


As we will see, this simple form of the probabilities $\dist_{ij}$ falls out naturally from generic optimization considerations. 
The fact that each entry is kept with probability proportional to its magnitude, besides being interesting on its own right, has a remarkably practical implication.
Every non-zero in the $i$-th row of $B$ will take the form $k_{ij}(\|A_{(i)}\|_1/s\rho_i)$ where $|k_{ij}|$ is the number times $A_{ij}$ was sampled.
Note that since we sample with replacement $|k_{ij}|$ might, in rare occasions, be more than $1$. 
The result is a matrix $B$ which is representable in $O(m\log(n) + s \log(n/s))$ bits. 
This is because there is no reason to store floating point matrix entry values. 
We use $O(m\log(n))$ bits to store all values $\|A_{(i)}\|_1/s\rho_i$ and $O(s \log(n/s))$ bits to store the non zero index \emph{offsets}.\footnote{It is harmless to assume any value in the matrix is kept using $O(\log(n))$ bits of precision. 
Otherwise, truncating the trailing bits can be shown to be negligible.} Note that $\sum |k_{ij}| = s$ and that some of these offsets might be zero. In a simple experiment, we measured the average number of bits per sample (total size of the sketch divided by the number of samples $s$). The results were between $5$ and $22$ bits per sample depending on the matrix and $s$. 
It is important to note that the number of bits per sample is usually less than $\log_2(n) + \log_2(m)$ which is the minimal number of bit required to represent a pair $(i,j)$. Our experiments show a reduction of disc space by a factor of between $2$ and $5$ relative to the \emph{compressed} size of the file representing the sample matrix $B$ in the standard row-column-value list format.

Another insight of our work is that the distributions we propose are combinations of two L1-based distributions. 
Which distribution is more dominant is determined by the sampling budget.
When the number of samples $s$ is small, $\rho_i$ is nearly linear in $\|A_{(i)}\|_1$ resulting in $\dist_{ij} \propto |A_{ij}|$.
However, as the number of samples grows, $\rho_i$ tends towards $\|A_{(i)}\|_1^2$ resulting in $\dist_{ij} \propto |A_{ij}| \cdot \|A_{(i)}\|_1$, a distribution we refer to as Row-L1 sampling. 
The dependence of the preferred distribution on the sample budget is also borne out in experiments, with sampling based on appropriately mixed distributions being consistently best.
This highlights that the need to adapt the sampling distribution to the sample budget is a genuine phenomenon. 

\section{Measure of Error and Related Work}\label{sec:past}

We measure the difference between $A$ and $B$ with respect to the L2 (spectral) norm as it is highly revealing in the context of data analysis. Let us define a \emph{linear trend} in the data of $A$ as any tendency of the rows to align with a particular unit vector $x$. To examine the presence of such a trend, we need only multiply $A$ with $x$: the $i$th coordinate of $A x$ is the projection of the $i$th row of $A$ onto $x$. Thus, $\|Ax\|_2$ measures the strength of linear trend $x$ in $A$, and $\|A\|_2$ measures the strongest linear trend in $A$. 
Thus, minimizing $\|A - B\|_2$ minimizes the strength of the strongest linear trend of $A$ {\em not captured} by $B$.
In contrast, measuring the difference using any entry-wise norm, e.g., the Frobenius norm, can be completely uninformative. 
This is because the best strategy would  be to always pick the largest $s$ matrix entries from $A$, a strategy that can easily be ``fooled". 
As a stark example, when the matrix entries are $A_{ij} \in \{0,1\}$, the quality of the approximation is \emph{completely independent} of which elements of $A$ we keep.
This is clearly bad; as long as $A$ contains even a modicum of structure certain approximations will be far better than others.

By using the spectral norm to measure error we get a natural and sophisticated target: to minimize $\|A-B\|_2$ is to make $E=A-B$ a near-rotation, having only small variations in the amount by which it stretches different vectors. This idea that the error matrix $E$ should be isotropic, thus packing as much Frobenius norm as possible for its L2 norm, motivated the first work on element-wise sampling of matrices by Achlioptas and McSherry~\cite{AchlioptasMcsherry2007}. Concretely, to minimize $\|E\|_2$ it is natural to aim for a matrix $E$ that is both zero-mean, i.e., an unbiased estimator of $A$, and whose entries are formed by sampling the entries of $A$ (and, thus, of $E$) independently. In the work of~\cite{AchlioptasMcsherry2007}, $E$ is a matrix of i.i.d.\ zero-mean random variables. The study of the spectral characteristics of such matrices goes  back all the way to Wigner's famous semi-circle law~\cite{semicircle}. Specifically, to bound $\|E\|_2$ in~\cite{AchlioptasMcsherry2007} a  bound due to Alon Krivelevich and Vu~\cite{kri_vu} was used,  a refinement of a bound by Juh\'{a}sz~\cite{r_g_spectrum} and F\"uredi and Koml\'os \cite{FK}. The most salient feature of that bound is that it depends on the \emph{maximum} entry-wise variance $\sigma^2$ of $A-B$, and therefore the distribution optimizing the bound is the one in which the variance of all entries in $E$ is the same. In turn, this means keeping each entry of $A$ independently with probability $\dist_{ij} \propto A_{ij}^2$ (up to a small wrinkle discussed below).

Several papers have since analyzed L2-sampling and variants \cite{nguyen2009matrix,nguyen2010tensor,DrineasZ11,gittens2009error,AchlioptasMcsherry2007}. An inherent difficulty of L2-sampling based strategies is the need for a special handling of small entries. This is because when each item $A_{ij}$ is kept with probability $p_{ij} \propto A^2_{ij}$, the resulting entry $B_{ij}$ in the sample matrix has magnitude $|A_{ij}/\dist_{ij}| \propto 1/|A_{ij}|$. Thus, if an extremely small element $A_{ij}$ is accidentally picked, the largest entry of the sample matrix ``blows up''. In \cite{AchlioptasMcsherry2007} this was addressed by sampling small entries with probability proportional to $|A_{ij}|$ rather than $A_{ij}^2$. 
In the work of Gittens and Tropp~ \cite{gittens2009error}, small entries are not handled separately and the bound derived depends on the ratio between the largest  and the smallest non-zero magnitude.  

Random matrix theory has witnessed dramatic progress in the last few years and~\cite{AhlswedeW02, RudelsonVershyninMatrixSampling2007, Tropp12, Recht2011} provide a good overview of the results. This progress motivated Drineas and Zouzias in~\cite{DrineasZ11} to revisit L2-sampling but now using concentration results for \emph{sums} of random matrices \cite{Recht2011}, as we do here. (Note that this is somewhat different from the original setting of~\cite{AchlioptasMcsherry2007} since now $E$ is not one random matrix with independent entries, but a sum of many independent matrices since the entries are chosen with replacement.) Their work improved upon all previous L2-based sampling results and also upon the L1-sampling result of Arora, Hazan and Kale~\cite{AroraHazanKale2006}, discussed below, while admitting a remarkably compact proof. The issue of small entries was handled in~\cite{DrineasZ11} by deterministically discarding all sufficiently small entries, a strategy that gives the strongest mathematical guarantee (but see the discussion regarding deterministic truncation in the experimental section).

A completely different tack at the problem, avoiding random matrix theory altogether, was taken by Arora et al.~\cite{AroraHazanKale2006}. Their approximation keeps the largest entries in $A$ deterministically (specifically all $A_{ij} \ge \eps/\sqrt{n}$ where the threshold $\eps$ needs be known a priori) and randomly rounds the remaining smaller entries to $\sign(A_{ij})\eps/\sqrt{n}$ or $0$. 
They exploit the simple fact  $\|A - B\| = \sup_{\|x\|=1,\|y\|=1} x^T(A-B)y$
by noting that as a scalar quantity its concentration around its expectation can be established by standard Bernstein-Bennet type inequalities. A union bound then allows them to prove that with high probability, $x^T(A-B)y \le \eps$ for {\it every} $x$ and $y$. The result of \cite{AroraHazanKale2006} admits a relatively simple proof. However, it also requires a truncation that depends on the desired approximation $\eps$. Rather interestingly, this time the truncation amounts to  keeping every entry \emph{larger} than some threshold. \

\section{Our Approach}\label{sec:approach}
Following the discussion in Section~\ref{sec:past} and in line with previous works, we:
(i) measure the quality of $B$ by $\|A-B\|_2$, 
(ii) sample the entries of $A$ independently, 
and (iii) require $B$ to be an unbiased estimator of $A$.
We are therefore left with the task of determining a good probability distribution $\dist_{ij}$ from which to sample the entries of $A$ in order to get $B$.
As discussed in Section~\ref{sec:past} prior art makes heavy use of beautiful results in the theory of random matrices. 
Specifically, each work proposes a specific sampling distribution and then uses results from random matrix theory to demonstrate that it has good properties.
In this work we reverse the approach, aiming for its logical conclusion. 
We start from a cornerstone result in random matrix theory and work backwards to reverse-engineer near-optimal distributions with respect to the notion of probabilistic deviations captured by the inequality.
The inequality we use it the Matrix-Bernstein inequality for sums of independent random matrices (see e.g., \cite{tropp2012user}, Theorem 1.6).
\begin{theorem}[Matrix Bernstein inequality]  \label{lem:MB} 
Consider a finite sequence $\{X_i\}$ of i.i.d.\ random $m \times n$ matrices, where $\E[X_1] = 0$ and $\|X_1\| \le R$. Let $\sigma^2 = \max \left\{   \| \E[X_1X_1^T]\|  ,\|\E[X_1^T X_1]\| \right\}$. 

For some fixed $s \ge 1$, let $X = (X_1+\cdots+X_s)/s$. For all $\eps \ge 0$,
\[
\Pr[\|X\| \ge \eps ] \le (m+n)\exp\left(-\frac{s\eps^2}{\sigma^2 + R\eps/3}\right) \enspace.
\]
\end{theorem}

To get a feeling for our approach, fix any probability distribution $p$ over the non-zero elements of $A$. 
Let $B$ be a random $m \times n$ matrix with exactly one non-zero element, formed by  sampling an element $A_{ij}$ of $A$ according to $\dist$ and  letting $B_{ij} = A_{ij}/\dist_{ij}$. 
Observe that for every $(i,j)$, regardless of the choice of $p$, we have $\E[B_{ij}] = A_{ij}$, and thus $B$ is always an unbiased estimator of $A$.  
Clearly, the same is true if we repeat this $s$ times taking i.i.d.\ samples  $B_1,\ldots,B_s$ and let our matrix $B$ be their average. 
With this approach in mind, the goal is now to find a distribution $\dist$ minimizing $\|E\|  = \|A-(B_1+\cdots+B_s)/s\|$.
Writing $sE = (A-B_1)+\cdots+(A-B_s)$ we see that $\|sE\|$ is the operator norm of a sum of i.i.d.\ zero-mean random matrices $X_i = A-B_i$, i.e., exactly the setting of Theorem~\ref{lem:MB}. The relevant parameters are
\begin{eqnarray}
\sigma^2  	& = & \max \left\{   \| \E[(A-B_1)(A-B_1)^T]\|  ,\|\E[(A-B_1)^T (A-B_1)]\| \right\} \label{eq:sigmadef} \\ 
R 			& = &  \max \|A-B_1\|_2 \;\;\;\text{ over all possible realizations of $B_1$} \enspace . \label{eq:Rdef}
\end{eqnarray}

Equations~\eqref{eq:sigmadef} and~\eqref{eq:Rdef} mark the starting point of our work. 
Our goal is to find probability distributions over the elements of $A$ that optimize~\eqref{eq:sigmadef} and~\eqref{eq:Rdef} \emph{simultaneously} with respect to their functional form in Theorem~\ref{lem:MB}, thus yielding the strongest possible bound on $\|A-B\|_2$. 
A conceptual contribution of our work is the discovery that these distributions \emph{depend} on the sample budget $s$, a fact also borne out in experiments.
The fact that minimizing the deviation metric of Theorem~\ref{lem:MB}, i.e., $\sigma^2 + R\epsilon/3$, suffices to bring out this non-linearity can be viewed as testament to the theorem's sharpness. 

Theorem~\ref{lem:MB} is stated as a bound on the probability that the norm of the error matrix is greater than some target error $\eps$ given the number of samples $s$. Nevertheless, in practice the target error $\eps$ is not known in advance, but rather is the quantity to minimize given the matrix $A$, the number of samples $s$, and the target confidence $\delta$. Specifically, for any given distribution $\dist$ on the elements of $A$, define
\begin{equation}\label{eps1}
\eps_1(\dist) =  \inf\left\{\eps : (m+n)\exp\left(-\frac{s\eps^2}{\sigma(\dist)^2 + R(\dist)\eps/3}\right)  \le \delta \right\} \enspace .
\end{equation}
Our goal in the rest of the paper is to seek the distribution $\dist^*$ minimizing $\eps_1$. Our result is an easily computable distribution $\dist$ which comes within a factor of 3 of $\eps_1(\dist^*)$ and, as a result, within a factor of 9 in terms of sample complexity (in practice we expect this to be even smaller, as the factor of 3 comes from consolidating bounds for a number of different worst-case matrices). To put this in perspective note that the definition of $\dist^*$ does not place \emph{any} restriction either on the access model for $A$ while computing $\dist^*$, or on the amount of time needed to compute $\dist^*$. In other words, we are competing against an oracle which in order to determine $\dist^*$ has \emph{all} of $A$ in its purview at once and can spend an unbounded amount of computation to determine it.

In contrast, the only global information regarding $A$ we will require are the \emph{ratios} between the L1 norms of the rows of the matrix. 
Trivially, the exact L1 norms of the rows (and therefore their ratios) can be computed in a single pass over the matrix, yielding a 2-pass algorithm. Moreover, standard concentration of measure arguments imply that these ratios can be estimated very well by sampling only a small number of columns. 
In our setting, it is in fact reasonable to expect that good estimates of these \emph{ratios} are  available a priori. 
This is because different rows correspond to different attributes and the ratios between the row norms reflect the ratios between the average absolute values of these features.
For example, if the matrix corresponds to text documents, knowing the ratios amounts to knowing global word frequencies. 
Moreover these ratios do not need to be known exactly to apply the algorithm, as even rough estimates of them give highly competitive results. 
Indeed, even disregarding this issue completely and simply assuming that all ratios equal $1$, yields an algorithm that appears quite competitive in practice, as demonstrated by our experiments.

\section{Data Matrices and Statement of Results}\label{sec:compar}

Throughout $A_{(i)}$ and $A^{(j)}$ will denote the $i$-th row and $j$-th column of $A$, respectively. Also, we use the  notation $\|A\|_1 = \sum_{i,j}|A_{ij}|$ and $\|A\|_F^2 = \sum_{i,j}A^2_{ij}$.
Before we formally state our result we introduce a definition that expresses the class of matrices for which our results hold.

\begin{definition}\label{DM}
An $m \times n$ matrix $A$ is a \emph{Data} matrix if:
\begin{enumerate}
\item
$\min_i \|A_{(i)}\|_1 \ge \max_j \|A^{(j)}\|_1$. \label{skewcond}
\item
$\displaystyle{\|A\|_1^2 / \|A\|_2^2 \ge 50 m}$. \label{rowcond}
\item
$m \ge 50$.\label{trivialcond}
\end{enumerate}
\end{definition}

Regarding Condition~\ref{skewcond}, recall that we think of $A$ as being generated by a measurement process of a fixed number of attributes (rows), each column corresponding to an observation. 
As a result, columns have bounded L1 norm, i.e., $\|A^{(j)}\|_1 \le \text{constant}$.
While this constant may depend on the type of object and its dimensionality, it is independent of the number of objects. 
On the other hand, $\|A_{(i)}\|_1$ grows linearly with the number of columns (objects). 
As a result, we can expect Definition~\ref{DM} to hold for all large enough data sets. 
Regarding Condition~\ref{rowcond}, it is easy to verify that unless the values of the entries of $A$ exhibit an unbounded variance, $\|A\|_1^2 / \|A\|_2^2$ grows as $\Omega(n)$ and Condition~\ref{rowcond} follows from $n \gg m$.
Condition~\ref{trivialcond} if trivial. 
Out of the three conditions the essential one is Condition~\ref{skewcond}. The other two are merely technical and hold in all non-trivial cases where Condition~\ref{skewcond} applies.

%
\begin{algorithm}[h!]
	\caption{Construct a sketch $B$ for a data matrix $A$}
	\label{alg:sketch}
	\begin{algorithmic}[1]
	\State {\bf Input:} Data matrix $A \in \R^{m \times n}$, sampling budget $s$, acceptable failure probability $\delta$ 
  	\State Set $\rho  \leftarrow$ \Call{ComputeRowDistribution}{$A$, $s$, $\delta$} \label{alg1}
	\State Sample $s$ elements of $A$ with replacement, each $A_{ij}$ having probability $\dist_{ij} = \rho_{i}\cdot |A_{ij}|/\|A_{(i)}\|_1$ \label{alg2}
	\State For each sample $\langle i, j , A_{ij}\rangle_\ell$, let entry $(i,j)$ of $B_{\ell}$ be $A_{ij}/\dist_{ij}$ and zero otherwise.  \label{alg3}
	\State {\bf Output:}  $B = \frac{1}{s}\sum_{\ell=1}^{s} B_{\ell}$. \label{alg4}
	\Statex \hrulefill
	\Function{ComputeRowDistribution}{$A$, $s$, $\delta$} \label{ComputeRowDistribution1}
		\State Obtain $z$ such that $z_i  \propto \|A_{(i)}\|_1$ for $i \in [m]$ \label{ComputeRowDistribution2}
      		\State Set $\alpha  \leftarrow \sqrt{\log((m+n)/\delta)/s}$ \;\;\; and \;\;\; $\beta \leftarrow \log((m+n)/\delta)/(3s)$ \label{ComputeRowDistribution3}
		\State Define $\rho_i (\zeta) = \left(\alpha z_i/2\zeta + \sqrt{ \left(\alpha z_i/2\zeta\right)^2 + \beta z_i/\zeta }\right)^2$ \label{ComputeRowDistribution4}
		\State Find $\zeta_1$ such that $\sum_{i=1}^{m}\rho_i(\zeta_1) = 1$ \label{ComputeRowDistribution5}
     		\State \textbf{return} $\rho$ such that $\rho_i = \rho_i(\zeta_1)$ for $i \in [m]$ \label{ComputeRowDistribution6}	            
	\EndFunction
	\end{algorithmic}
\end{algorithm}
To simplify the exposition of algorithm \ref{alg:sketch}, we describe it in a the \emph{non-streaming} setting. 
About the complexity of Algorithm~\ref{alg:sketch}, steps~\ref{ComputeRowDistribution1}--\ref{ComputeRowDistribution6} compute a distribution $\rho$ over the rows.
Assuming step~\ref{ComputeRowDistribution2} can be implemented efficiently (or skipped altogether, see discussion at the bottom of Section~\ref{sec:approach}) 
the running time of $\operatorname{ComputeRowDistribution}$ is independent of $n$.
Finding $\zeta_1$ in step~\ref{ComputeRowDistribution5} can be done very efficiently by binary search because the function $\sum_{i}\rho_i(\zeta)$ is strictly decreasing in $\zeta$.
Conceptually, we see that the probability assigned to each element $A_{ij}$ in Step~\ref{alg2} is simply the probability $\rho_i$ of its row times its intra-row weight $|A_{ij}|/\|A_{(i)}\|_1$. 

Note that to apply Algorithm~\ref{alg:sketch} the entries of $A$ must be sampled \emph{with} replacement in the streaming model.
A simple way to achieve this using $O(s)$ operations per matrix element and $O(s)$ active memory was presented in~\cite{DrineasKMW2006}. 
In fact, though, it is possible to implement such sampling far more efficiently.
\begin{theorem} \label{thm:algworks}
For any matrix $A$, steps~\ref{alg2}-\ref{alg4} in Algorithm~\ref{alg:sketch} can be accomplished 
using $O(\log(s))$ active memory, $\tilde{O}(s)$ space, and $O(1)$ operations per non zero element of $A$ in the streaming model. 
\end{theorem}


\noindent  We are now able to state our main result. 
\begin{theorem}\label{thm:main}
If $A$ is a Data matrix (per Definition~\ref{DM}) and $\dist$ is the probability distribution defined in Algorithm~\ref{alg:sketch}, then $\eps_1(\dist) \le 3 \, \eps_1(\dist^*)$, where $\dist^*$ is the minimizer of $\eps_1$.
\end{theorem}

The proof of Theorem~\ref{thm:main} is outlined in Section~\ref{sec:mainproof}.
To understand the implications of Theorem~\ref{thm:main} and to compare our result with previous ones we must first define several matrix metrics. 

\noindent {\bf Stable rank}: Denoted as $\numrank$ and defined as $\|A\|_F^2/\|A\|_2^2$. This is a smooth analog for the algebraic rank, always bounded by it from above, and resilient to small perturbations of the matrix. 
For data matrices we expect it to be small (even constant) and to capture the ``inherent dimensionality" of the data.

\noindent  {\bf Numeric density}: Denoted as $\numdensity$ and defined as $\|A\|_1^2/\|A\|_F^2$, this is a smooth analog of the number of non-zero entries $\nnz(A)$. For 0-1 matrices it equals $\nnz(A)$, but when there is variance in the magnitude of the entries it is smaller.

\noindent  {\bf Numeric row density}: Denoted as $\numrowdensity$ and defined as $\sum_i \|A_{(i)}\|_1^2 / \|A\|_F^2 \le n$. In practice, it is often close to the average numeric density of a single row, a quantity  typically much smaller than $n$.

\begin{theorem} \label{thm:s as eps}
Let $A$ be a data matrix meeting the conditions of Definition~\ref{DM}. 
Let $B$ be the matrix returned by Algorithm~\ref{alg:sketch} for $\eps>0$ and 
\[
s \ge s_0 = \Theta( \numrowdensity \cdot \numrank/\eps^2 \cdot \log(n/\delta) + ( \numrank \cdot \numdensity / \eps^2 \cdot \log(n/\delta)  )^{1/2} ) \enspace .
\]
Then $\|A-B\| \leq \eps \|A\|$ with probability at least $1 -\delta$.
\end{theorem}

The table below shows the corresponding number of samples in previous works for constant success probability, in terms of the matrix metrics defined above. 
The fourth column presents the ratio of the samples needed by previous results divided by the samples needed by our method. 
To simplify the expressions, we present the ratio between our bound and~\cite{AroraHazanKale2006} only when the result of~\cite{AroraHazanKale2006} gives superior bounds to~\cite{DrineasZ11}.
That is, we always compare our bound to the stronger of the two bounds implied by these works. 

\begin{table}[h!]
\begin{center}
\renewcommand{\arraystretch}{1.4}
\begin{tabular}{| C{1.7cm}  | C{1.7cm} | C{5cm}  | c|} \hline
Citation 					& Method 		& Number of samples needed 							&  Improvement ratio of Theorem~\ref{thm:s as eps} 							\\  \hline 
\cite{AchlioptasMcsherry2007}  	& L1, L2 		& $\numrank \cdot (n /\eps^2) + n \cdot \mathrm{polylog}(n)$	&  																\\  \hline
\cite{DrineasZ11} 			& L2 			& $ \numrank \cdot (n/\eps^2) \log(n)$ 					& $\numrowdensity/n + (\sqrt{\numdensity}/n)\cdot (\eps/\sqrt{\numrank \log(n)})  $ 	\\  \hline
\cite{AroraHazanKale2006}  	& L1 			& $ (\numdensity \cdot \, n/\eps^2)^{1/2}$ 					& $\sqrt{\numrank \cdot \log(n) /n}$ 										\\  \hline
This paper 				& Bernstein	& $ \numrowdensity \cdot \numrank/\eps^2 \cdot \log(n) +   ( \numrank \cdot \numdensity / \eps^2 \cdot \log(n)  )^{1/2}  $ & 						\\  \hline
\end{tabular}
\end{center}
\end{table}
 
Holding $\eps$ and the stable rank constant we readily see that our method requires roughly $1/\sqrt{n}$ the samples needed by~\cite{AroraHazanKale2006}. 
In the comparison with~\cite{DrineasZ11}, the key parameter is the ratio $\numrowdensity /  n$. 
This quantity is typically much smaller than $1$ for data matrices but independent of $n$. 
As a point of reference for the assumptions, in the experimental Section~\ref{sec:experiments} we provide the values of all relevant matrix metrics for all the real data matrices we worked with, wherein the ratio
$\numrowdensity /  n$ is typically around $10^{-2}$.
Considering this, one would expect that L2-sampling should experimentally fare better than L1-sampling. 
As we will see, quite the opposite is true. A potential explanation for this phenomenon is the relative looseness of the bound of~\cite{AroraHazanKale2006} for the performance of L1-sampling.

\section{Proof of Theorem~\ref{thm:s as eps}} \label{sec:mainproof}

We start by iteratively replacing the objective functions~\eqref{eq:sigmadef} and~\eqref{eq:Rdef} with increasingly simpler functions. 
Each replacement will incur a (small) loss in accuracy but will bring us closer to a function for which we can give a closed form solution.  
Recalling the definitions of $\alpha,\beta$ from Algorithm~\ref{alg:sketch} and rewriting the requirement in~\eqref{eps1} as a quadratic form in $\eps$ gives $\eps^2 - \eps   \beta R - (\alpha\sigma)^2 > 0$. Our first step is to observe that for any $c,d > 0$, the equation $\eps^2  - \eps \cdot c - d = 0$ has one negative and one positive solution and that the latter is at least $(c+\sqrt{d})/\sqrt{2}$ and at most $c+\sqrt{d}$. Therefore, if we define\footnote{Here and in the following, to lighten notation, we will omit all arguments, i.e., $\dist,\sigma(\dist),R(\dist)$, from the objective functions $\eps_i$ we seeks to optimize, as they are readily understood from context.}
$\eps_2  \eqdef  \alpha \sigma + \beta R$ we see that $1/\sqrt{2} \le \eps_1/\eps_2 \le 1$.
Our next simplification encompasses Conditions~\ref{trivialcond},~\ref{rowcond} of Definition~\ref{DM}. Let $\eps_3 \eqdef \alpha \tsig + \beta  \tR$ where
$$
\tsig^2 \eqdef  \max\left\{  \max_i \sum_j A_{ij}^2/\dist_{ij} \; , \; \max_j \sum_i A_{ij}^2/\dist_{ij} \right\}  \quad
\tR  \eqdef  \max_{ij} |A_{ij}|/\dist_{ij} \enspace .
$$
\begin{lemma}\label{lem:degen}
For every matrix $A$ satisfying Conditions~\ref{trivialcond} and~\ref{rowcond} of Definition~\ref{DM}, for every probability distribution on the elements of $A$, $|\eps_2/\eps_3 -1| \le 1/50$.
\end{lemma}

\begin{proof}

We start by providing an two auxiliary lemmas.

\begin{lemma} \label{lem:optd}	
For any $x,p \in \R^n$, if $p_i \ge 0$ and $\|p\|_1=1$, then $\max_k |x_k|/p_k \ge  \|x\|_1$ and $\sum_k  x_k^2/p_k 	\ge  \|x\|_1^2$, with equality holding in both cases  if and only if $p_k = |x_k| /\|x\|_1$.
\end{lemma}
\begin{proof}
To prove $\max_k |x_k|/p_k \ge  \|x\|_1$ we note that if $|x_i|/p_i  \neq |x_j|/p_j$, then changing $p_i,p_j$ to $p'_i,p_j'$ such that $p_i'+p_j'=p_i+p_j$ and $|x_i|/p'_i  = |x_j|/p'_j$ can only reduce the maximum. In order for all $|x_k|/p_k$ to be equal it must be that $p_k = |x_k| /\|x\|_1$ for all $j$, in which case $\max_k |x_k|/p_k = \|x\|_1$.

The second claim follows from applying Jensen's inequality to the convex function $x \mapsto x^2$. Specifically, Jensen's inequality shows that for any $p$, 
$$ \E_{i \sim p}[(|x_i|/p_i)^2] \geq \E_{i \sim p}[(|x_i|/p_i)]^2 = \|x\|_1^2 $$
This inequality is met for $p_i = |x_i| /\|x\|_1$.
\end{proof}

\begin{lemma} \label{lem:tsig}
For any matrix $A$ and any probability distribution $\dist$ on the elements of $A$, we have 
$|\sig^2/\tsig^2 - 1| \leq \frac{\|A\|_2^2}{\sum_i  \|A_{(i)}\|_1^2}$ and $|R/\tR - 1| \leq \frac{\|A\|_2}{\|A\|_1}$.
\end{lemma}
\begin{proof}
Recall that $B_1$ contains one non-zero element $A_{ij}/\dist_{ij}$, while all its other entries are 0. Therefore, $\E[B_1B_1^T]$ and $\E[B_1^TB_1]$ are both diagonal matrices where
\[
\E[(B_1B_1^T)_{i,i}]
= \sum_j A_{ij}^2/\dist_{ij}  \qquad \mbox{and} \qquad \E[(B_1^TB_1)_{j,j}]
= \sum_i A_{ij}^2/\dist_{ij} \enspace .
\]
Since the operator norm of a diagonal matrix equals its largest entry we see that 
\[
\tsig^2 \eqdef  \max\left\{  \max_i \sum_j A_{ij}^2/\dist_{ij} \; , \; \max_j \sum_i A_{ij}^2/\dist_{ij} \right\}  =\max\{\|\E[B_1 B_1^T]\|,\|\E[B_1^T B_1]\|  \} \enspace .
\]

We will need to bound $\tsig^2$ from below. Trivially, $\tsig^2 \ge \|\E[B_1 B_1^T]\| = \max_i\sum_j A_{ij}^2/\dist_{ij}$. 
Defining $\rho_i \eqdef \sum_j p_{ij}$ and $q_{ij}\eqdef p_{ij}/\rho_i$, the below second and third inequalities are a result of Lemma~\ref{lem:optd}
\begin{equation}\label{eq:denseprod}
\tsig^2 \geq \max_i \sum_j \frac{A_{ij}^2}{\dist_{ij}} = \max_i  \rho_i^{-1} \sum_j \frac{A_{ij}^2}{q_{ij}} \ge \max_i \rho_i^{-1}\|A_{(i)}\|_1^2 \ge \sum_i  \|A_{(i)}\|_1^2   \enspace .
\end{equation}
On the other hand, $\sigma^2=\max\{\|\E[Z_1 Z_1^T]\|,\|\E[Z_1^T Z_1]\|  \}$, where $Z_1 = B_1-A$. Since $\E[B_1]=A$,
\[
 \|\E[Z_1Z_1^T]\| = \|\E[B_1B_1^T - AB_1^T - B_1A^T + AA^T]\| = \|\E[B_1B_1^T]-AA^T\| 
\]
and, analogously,  $\|\E[Z_1^TZ_1]\| =  \|\E[B_1^TB_1]-A^TA\|$. Therefore, by the triangle inequality, $|\sig^2 - \tsig^2| \leq \|A\|^2$ and the claim now follows from~\eqref{eq:denseprod}. 

Recall that $B_1$ contains one non-zero entry $A_{ij}/\dist_{ij}$ and that $R$ is the maximum of  $\|B_1-A\|$ over all possible realizations of $\dist$, i.e., choices of $(i,j)$. Thus by the triangle inequality,  
$$ R = \max \|B_1 - A\| \leq \max \|B_1\| + \|A\| \quad \text{and} \quad  R \geq \max\|B_1\|-\|A\| \enspace .$$
Since $B_1$ has one non-zero entry, we see that $\max \|B_1\|_2 = \max_{ij} |A_{ij}|/\dist_{ij} = \tR$ and, thus, $|R/\tR - 1| \leq \|A\|_2/\tR$. Applying Lemma~\ref{lem:optd} to $A \in \R^{m \times n}$ with distribution $\dist$  yields $\tR \geq \|A\|_1$. 
\end{proof}

We are now ready to prove lemma~\ref{lem:degen}. It suffices to prove that both $|\sig^2/\tsig^2 - 1|$ and $|R/\tR - 1|$ are bounded by $1/50$.
Lemma~\ref{lem:tsig} yields the first inequality below and Condition~\ref{rowcond} of Definition~\ref{DM} the second. The third inequality holds for every matrix $A$, with equality occurring when all rows have the same L1 norm.
\[
|\sig^2/\tsig^2 - 1| \leq \frac{\|A\|_2^2}{\sum_i  \|A_{(i)}\|_1^2} \le \frac{\|A\|_1^2}{50 m \sum_i  \|A_{(i)}\|_1^2}  \le \frac{1}{50}  \enspace .
\]
Lemma~\ref{lem:tsig} yields the first inequality below. The second inequality follows from rearranging the factors in the second inequality above. 
Condition~\ref{trivialcond} of Definition~\ref{DM}, i.e., $m \ge 50$, implies the third. 
\[
|R/\tR - 1| \leq \frac{\|A\|_2}{\|A\|_1} \le \frac{1}{\sqrt{50 m}} \le
 \frac{1}{50} \enspace .
\]
\end{proof}

This allows us to optimize $\dist$ with respect to $\eps_3$ instead of $\eps_2$.
In minimizing $\eps_3$ we see that there is freedom to use different rows to optimize $\tsig$ and $\tR$. At a cost of a factor of 2, we will couple the two minimizations by minimizing $\eps_4 = \max\{\eps_5,\eps_6\}$ where
\begin{eqnarray}
\eps_5 \eqdef  	\max_i 	\left[\alpha \sqrt{\sum_j \frac{A_{ij}^2}{\dist_{ij}}} 	+ \beta\max_{j}  \frac{|A_{ij}|}{\dist_{ij}}  \right], \qquad
\eps_6 \eqdef 	\max_j  	\left[\alpha \sqrt{\sum_i \frac{A_{ij}^2}{\dist_{ij}}} 	+ \beta\max_{i}  \frac{|A_{ij}|}{\dist_{ij}} 	\right]  \enspace .
\end{eqnarray}
Note that the maximization of $\tR$ in $\eps_5$ (and $\eps_6$)  is coupled with that of the $\tsig$-related term by constraining the optimization to consider only one row (column) at a time. Clearly, $1 \le \eps_3/\eps_4 \le 2$.

Next we focus on $\eps_5$, the first term in the maximization of $\eps_4$. We first present a lemma analyzing the distribution minimizing it. The lemma provides two important insights. First, it leads to an efficient algorithm for finding a distribution minimizing $\eps_5$ and second, it is key in proving that for all data matrices satisfying Condition~\ref{skewcond} of Definition~\ref{DM}, by minimizing $\eps_5$ we also minimize $\eps_4 = \max\{\eps_5,\eps_6\}$.
\begin{lemma}\label{optDist}
A minimizer to the function $\eps_5$ can be found, to precision $\eta$ in time logarithmic in $\eta$. Specifically the function $\eps_5$ is minimized by $\dist_{ij} =  \rho_i q_{ij}$ where $q_{ij} = |A_{ij}|/\|A_{(i)}\|_1$.
To define $\rho_i$ let $z_i  \propto \|A_{(i)}\|_1$ and define 
$\rho_i (\zeta) = \left(\alpha z_i/2\zeta + \sqrt{ \left(\alpha z_i/2\zeta\right)^2 + \beta z_i/\zeta }\right)^2$. 
Let $\zeta_1>0$ be the unique solution to\footnote{Notice that the function $\sum \rho_i(\zeta)$ is monotonically decreasing for $\zeta>0$ hence the solution is indeed unique, and can be found via a binary search.} $\sum_i \rho_i(\zeta_1)=1$. We set $\rho_i \eqdef \rho_i(\zeta_1)$.
\end{lemma}

\begin{proof}
To find the probability distribution $\dist$ that minimizes $\eps_5$ we start by writing $\dist = \rho_i q_{ij}$, without loss of generality. That is, we decompose $\dist$ to a distribution $\rho_i \ge 0$  over the rows of the matrix, i.e., $\sum_i \rho_i = 1$, and a distribution $q_{ij} \ge 0$ within each row $i$, i.e., $\sum_j q_{ij} = 1$, for all $i$.  We first prove that (surprisingly) the optimal $q$ has a closed form solution while the optimal $\rho$ is efficiently computable.

For any $\rho$, writing~$\eps_5$ in terms of $\rho_i, q_{ij}$ we see that $\eps_5$ is the maximum, over rows $1 \le i \le m$ , of
\begin{equation}\label{eps5_crap}
\frac{\alpha}{\sqrt{\rho_i}} \sqrt{\sum_j\frac{ A_{ij}^2}{q_{ij}}} + \frac{\beta}{\rho_i} \max_{j} \frac{|A_{ij}|}{q_{ij}} \enspace .
\end{equation}

Observe that since $\rho$ is fixed, the only variables in the above expression for each row  $i$ are the $q_{ij}$. Lemma~\ref{lem:optd} implies that setting $q_{ij} = |A_{ij}|/\|A_{(i)}\|_1$ simultaneously minimizes both terms in~\eqref{eps5_crap}. This means that for \emph{every} fixed probability distribution $\rho$, the minimizer of $\eps_5$ satisfies $q_{ij} = \frac{|A_{ij}|}{\|A_{(i)}\|_1}$. Thus, we are left to determine 
\[
\Phi(\rho) = \max_i \left[\frac{  \alpha \|A_{(i)}\|_1}{\sqrt{\rho_i}} +  \frac{\beta \|A_{(i)}\|_1}{ \rho_i} \right] \enspace .
\]

Unlike the intrarow optimization, the two summands in $\Phi$ achieve their respective minima at different distributions $\rho$.
To get some insight into the tradeoff, let us first consider the two extreme cases. When $\beta = 0$, minimizing the maximum over $i$ requires equating all $\|A_{(i)}\|_1/\sqrt{\rho_i}$, i.e., $\rho_i \propto \|A_{(i)}\|_1^2$, leading to the distribution we call ``row-$L_1$", i.e., $\dist_{ij} \; \propto \; |A_{ij}| \cdot \|A_{(i)}\|_1 $. When $\alpha = 0$, equating the $\|A_{(i)}\|_1/ \rho_i$ requires  $\rho_i \; \propto \; \|A_{(i)}\|_1$, leading to the ``plain-$L_1$" distribution $\dist_{ij}  \; \propto \;|A_{ij}|$.
Nevertheless, since we wish to minimize the maximum over several functions, we can seek $\dist$ under which all functions are equal, i.e., such that there exists $\zeta >0$ such that for all $i$,
\[
 \frac{\alpha \|A_{(i)}\|_1}{\sqrt{\rho_i}} +  \frac{\beta \|A_{(i)}\|_1}{ \rho_i} = \zeta>0 \enspace .
\]  
Solving the resulting quadratic equation and selecting for the positive root yields equation~\eqref{de:rhoi}, i.e.,
\begin{equation} \label{de:rhoi} 
\rho_i (\zeta) = \left(\frac{\alpha \|A_{(i)}\|_1}{2\zeta} + \sqrt{ \left(\frac{\alpha \|A_{(i)}\|_1}{2\zeta}\right)^2 + \frac{\beta \|A_{(i)}\|_1}{\zeta} }\right)^2 \enspace .
\end{equation}
Since the quantities under the square root in~\eqref{de:rhoi} are all positive we see that it is always possible to find $\zeta >0$ such that all equalities hold, and thus~\eqref{de:rhoi} does minimize $\eps_5$ for every matrix $A$. 
Moreover, since the right hand side of~\eqref{de:rhoi} is strictly decreasing in $\zeta$, binary search finds the unique value of $\zeta$ such that $\sum \rho_i =1$.
\end{proof}

We now prove that in order to minimize $\eps_4$ and thus approximately minimize the original function $\eps$, it suffices, under appropriate conditions, to minimize $\eps_5$.

\begin{lemma}\label{skew}
For every matrix satisfying Condition~\ref{skewcond} of Definition~\ref{DM}, 
$\mathrm{argmin}_\dist \ \eps_5 \subseteq \mathrm{argmin}_\dist \ \eps_4$.
\end{lemma}

\begin{proof}
We begin with an auxiliary lemma.

\begin{lemma}
For any two functions $f,g$, if $x_0 = \arg\min_x f(x)$ and $g(x_0) \le f(x_0)$, then $\min_x \max\{f(x),g(x)\} = f(x_0)$.
\end{lemma}
\begin{proof} 
$\min_x \max\{f(x),g(x)\} \ge \min_x f(x) = f(x_0) = \max\{f(x_0),g(x_0)\} \ge \min_x \max\{f(x),g(x)\}$
\end{proof}

Thus, it suffices to evaluate $\eps_6$ at the distribution $\dist$ minimizing $\eps_5$ and check that $\eps_6(\dist) \le \eps_5(\dist)$. 
We know that $\dist$ is of the form $\dist_{ij} = \rho_i |A_{ij}|/\|A_{(i)}\|_1 $ for some distribution $\rho$. Substituting this form of $\dist$ into $\eps_6$ gives~\eqref{eq:subsprod}. Condition~\ref{skewcond} of Lemma~\ref{DM}, i.e.,  $\max_j \|A^{(j)}\|_1 \le \min_i \|A_{(i)}\|_1$, allows us to pass from~\eqref{eq:skewderiv} to~\eqref{eq:skewderiva}. Finally, to pass from~\eqref{eq:skewderiva} to~\eqref{eq:skewderivb} we note that the two maximizations over $i$ in~\eqref{eq:skewderiva} involve the same expression, thus externalizing the maximization has no effect.
\begin{eqnarray}
\eps_6(\dist) &=& \max_j 
\left[ 
\alpha \left({\sum_i \frac{\|A_{(i)}\|_1 \cdot |A_{ij}|}{\rho_i}}\right)^{1/2} + \beta \max_i \frac{\|A_{(i)}\|_1}{\rho_i} 
\right] \label{eq:subsprod}\\
&\le& 
\max_j 
\left[ 
\alpha 
\left(
\max_i 
\frac{\|A_{(i)}\|_1}{\rho_i} 
\cdot \sum_i |A_{ij}| \right)^{1/2} + \beta \max_i \frac{\|A_{(i)}\|_1}{\rho_i} \right] \nonumber  \\ 
&=& 
\max_j 
\left[ 
\alpha 
\left(
\max_i 
\frac{\|A_{(i)}\|_1}{\rho_i} 
\cdot \|A^{(j)}\|_1 \right)^{1/2} + \beta \max_i \frac{\|A_{(i)}\|_1}{\rho_i} \right] \nonumber  \\
&\le&  
\alpha \left(\max_i 
\frac{\|A_{(i)}\|_1}{\rho_i} 
\cdot \max_j \|A^{(j)}\|_1 \right)^{1/2} + \beta \max_i \frac{\|A_{(i)}\|_1}{\rho_i} \label{eq:skewderiv} \\ 
&\le&  \alpha \left(\max_i \frac{\|A_{(i)}\|_1}{\rho_i} \cdot \min_i \|A_{(i)}\|_1 \right)^{1/2} + \beta \max_i \frac{\|A_{(i)}\|_1}{\rho_i} \label{eq:skewderiva} \\ 
&\le& \max_i\left[ \alpha \left(\frac{\|A_{(i)}\|_1}{\rho_i} \cdot \min_i \|A_{(i)}\|_1 \right)^{1/2} + \beta \frac{\|A_{(i)}\|_1}{\rho_i}  \right] \label{eq:skewderivb}\\ 
&\le& \max_i\left[ \alpha \frac{\|A_{(i)}\|_1}{\sqrt{\rho_i}}  + \beta \max_i \frac{\|A_{(i)}\|_1}{\rho_i} \right] \nonumber \\
& = & \eps_5(\dist) \nonumber \enspace .
\end{eqnarray}

\end{proof}

\noindent We are now ready to prove our main Theorem.
\begin{proof} [Proof of Theorem \ref{thm:main}]
Recall that above we proved that
$$1/\sqrt{2} \leq \eps_1/\eps_2 \leq 1, \ \ \ 49/50 \leq \eps_2/\eps_3 \leq 51/50, \ \ \ 1 \leq \eps_3/\eps_4 \leq 2 .$$
Let $p$ be a minimizer of $\eps_5$. According to Lemma~\ref{skew}, it is also a minimizer of $\eps_4$.
Hence, it holds that 
$$ \eps(p)/\min_q\{ \eps(q) \} \leq \sqrt{2} \cdot \frac{50}{49} \cdot \frac{51}{50} \cdot 2 \leq 3$$
as required
\end{proof}

We finish with the proof of Theorem~\ref{thm:s as eps} analyzing the value of $\eps(p)$.
 
\begin{proof}[Proof of Theorem \ref{thm:s as eps}]

We start by computing the value of $\eps_1$ as a function of $s,\delta$, for the probability distribution $P_0$ minimizing $\eps_5$. Recall that in deriving~\eqref{de:rhoi} we established that $\eps_5(P_0) = \zeta_0$, where $\zeta_0$ is such that $\sum_{i=1}^m \rho_i(\zeta_0)=1$, i.e.,
\begin{equation}\label{zohoh}
1 = \sum_{i=1}^m \left(\frac{\alpha \|A_{(i)}\|_1}{2\zeta_0} + \sqrt{ \left(\frac{\alpha \|A_{(i)}\|_1}{2\zeta_0}\right)^2 + \frac{\beta \|A_{(i)}\|_1}{\zeta_0} }\right)^2 \leq \sum_{i=1}^m \frac{\alpha^2 \|A_{(i)}\|_1^2}{\zeta_0^2} + \frac{2\beta \|A_{(i)}\|_1}{\zeta_0}  \enspace .
\end{equation}
This yields the following quadratic equation in $\zeta_0$
\begin{equation}\label{zohho}
\zeta_0^2 - \zeta_0 \cdot 2\beta \|A\|_1 - \alpha^2 \sum_i  \|A_{(i)}\|_1^2 \leq 1
\end{equation}
Treating~\eqref{zohho} as an equality and bounding the larger root of the resulting quadratic equation we get
\begin{equation}\label{zohhoo}
\zeta_0 = O\left(\beta \|A\|_1 +  \alpha \sqrt{\sum_i  \|A_{(i)}\|_1^2}  \right) = 
O\left( \frac{\log\left(\tfrac{m+n}{\delta}\right)\|A\|_1}{s} + \sqrt{\frac{\log\left(\tfrac{m+n}{\delta}\right)\sum_i \|A_{(i)}\|_1^2}{s}} \right)  
\end{equation}
The second equality is obtain by replacing $\alpha,\beta$ with their corresponding expressions in Algorithm~\ref{alg:sketch}, line~\ref{ComputeRowDistribution3}: $\alpha = \sqrt{\log((m+n)/\delta)/s}$ and $\beta=\log((m+n)/\delta)/(3s)$.
Recall that to prove Theorem~\ref{thm:main} we proved that if $A$ meets the conditions of Definition~\ref{DM}, then
\[
\min_P \eps_1(\dist) = \Theta(\zeta_0) \enspace . 
\]
It follows that for $\eps^* = \min_P \eps_1(\dist)$,
$$s = O\left(\frac{\log((m+n)/\delta)\sum_i \|A_{(i)}\|_1}{\eps^*} + \frac{\log((m+n)/\delta)\sum_i \|A_{(i)}\|_1^2}{(\eps^*)^2} \right) $$
The theorem now follows by taking $\eps^* = \eps\|A\|$.
\end{proof}


\section{Experiments} \label{sec:experiments}
We experimented with $4$ matrices with different characteristics, these are summarized in the table below. See Section~\ref{sec:compar} for the definition of the different characteristics.
\begin{table}[h!]
\begin{center}
\begin{tabular}{|C{1.4cm} || C{0.9cm} | C{0.9cm} | C{0.9cm} | C{0.9cm} | C{0.9cm} | C{0.9cm} | C{0.9cm} | C{0.9cm} | c |} \hline 
Measure	& $m$	& $n$	& $\nnz(A)$ 	& $\|A\|_1$	& $\|A\|_F$	&$\|A\|_2$		& $\numrank$	& $\numdensity$	& $\numrowdensity$ \\ \hline \hline
Synthetic	& 1.0e+2	& 1.0e+4	& 5.0e+5		& 1.8e+7		& 3.2e+4		& 8.7e+3		& 1.3e+1		& 3.1e+5			& 3.2e+3 \\ \hline
Enron	& 1.3e+4	& 1.8e+5	& 7.2e+5		& 4.0e+9		& 5.8e+6		& 1.0e+6		& 3.2e+1		& 4.9e+5			& 1.5e+3 \\ \hline
Images	& 5.1e+3	& 4.9e+5	& 2.5e+8		& 6.5e+9		& 2.0e+6		& 1.8e+6		& 1.3e+0		& 1.1e+7			& 2.3e+3 \\ \hline
Wikipedia	& 4.4e+5	& 3.4e+6	& 5.3e+8		& 5.3e+9		& 7.5e+5		& 1.6e+5		& 2.1e+1 		& 5.0e+7			& 1.9e+4	\\ \hline
\end{tabular}
\end{center}
\end{table}

\vspace{-.3cm}
\noindent {\bf Enron:} Subject lines of emails in the Enron email corpus~\cite{Sty11}. Columns correspond to subject lines, rows to words, and entries to tf-idf values. 
This matrix is extremely sparse to begin with.\\
\noindent {\bf Wikipedia:} Term-document matrix of a fragment of Wikipedia in English. Entries are tf-idf values.\\
\noindent {\bf Images:} A collection of images of buildings from Oxford \cite{Philbin07}. Each column represents the wavelet transform of a single $128\times 128$ pixel grayscale image.\\
\noindent {\bf Synthetic:} This synthetic matrix simulates a collaborative filtering matrix. Each row corresponds to an item and each column to a user. Each user and each item was first assigned a random latent vector (i.i.d. Gaussian). Each value in the matrix is the dot product of the corresponding latent vectors plus additional Gaussian noise. We simulated the fact that some items are more popular than others by retaining each entry of each item $i$ with probability $1-i/m$ where $i = 0,\ldots,m-1$.

\subsection{Sampling techniques and quality measure}\label{samplings}

The experiments report the accuracy of sampling according to four different distributions. 
In Figure~\ref{plotsplots}, {\bf Bernstein} denotes the distribution of this paper, defined in Lemma~\ref{optDist}. 
The {\bf Row-L1} distribution is a simplified version of the Bernstein distribution, where $\dist_{ij} \propto |A_{ij}| \cdot \|A_{(i)}\|_1$.
{\bf L1} and {\bf L2} refer to $\dist_{ij} \propto |A_{ij}|$ and $\dist_{ij} \propto |A_{ij}|^2$, respectively, as defined earlier in the paper. The case of {\bf L2} sampling was split into three sampling methods corresponding to different trimming thresholds. In the method referred to as {\bf L2} no trimming is made and $\dist_{ij} \propto |A_{ij}|^2$. In the case referred to as {\bf L2 trim 0.1}, $\dist_{ij} \propto |A_{ij}|^2$ for any entry where $|A_{ij}|^2 > 0.1 \cdot \E_{ij}[|A_{ij}|^2]$ and $\dist_{ij}=0$ otherwise. The sampling technique referred to as {\bf L2 trim 0.01} is analogous with threshold $0.01 \cdot \E_{ij}[|A_{ij}|^2]$.


Although to derive our sampling probability distributions we targeted minimizing $\|A-B\|_2$, in experiments it is
more informative to consider a more sensitive measure of quality of approximation. The reason is that, due to scaling, for a number of values of $s$ one has $\|A-B\|_2 > \|A\|_2$ which would suggest that the all zeros matrix is a better sketch for $A$ than the sampled matrix. 
We will see that this is far from being the case. As a trivial example, consider the possibility $B \approx 10A$. 
Clearly, $B$ is very informative of $A$ although  $\|A-B\| \ge 9\|A\|$. 
To avoid this pitfall, we measure $\|P_{k}^B A \|_F/\|A_k\|_F$, where $P_k^{B}$ is the projection on the top $k$ left singular vectors of $B$. 
Here, $A_k = P_k^A A$ is the optimal rank $k$ approximation of $A$. 
Intuitively, this measures how well the top $k$ left singular vectors of $B$ capture $A$, compared to $A$'s own top-$k$ left singular vectors. 
We also compute $\|A Q_k^B\|_F/\|A_k\|_F$ where $Q_k^{B}$ is the projection on the top $k$ right singular vectors of $A$. 
Note that, for a given $k$, approximating the row-space is harder than approximating the column-space since it is of dimension $n$ which is significantly larger than $m$, a fact also borne out in the experiments. In the experiments we made sure to choose a sufficiently wide range of sample sizes so that at least the best method for each matrix goes from poor to near-perfect both in approximating the row and the column space. 
In all cases we report on $k=20$ which is close to the upper end of what could be efficiently computed on a single machine for matrices of this size. The results for all smaller values of $k$ are qualitatively indistinguishable.

\newcommand{\myfigurewidth}{7.5cm}
\newcommand{\myfigurehight}{5cm}

\begin{figure}[htbp]
\label{results_figure}
\begin{center}
\includegraphics[width=14cm]{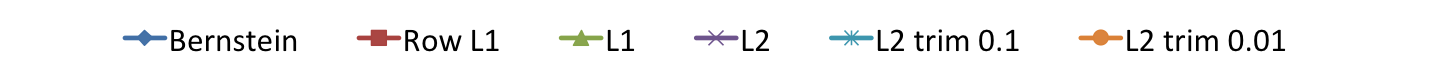} \\
\includegraphics[width=\myfigurewidth , height=\myfigurehight]{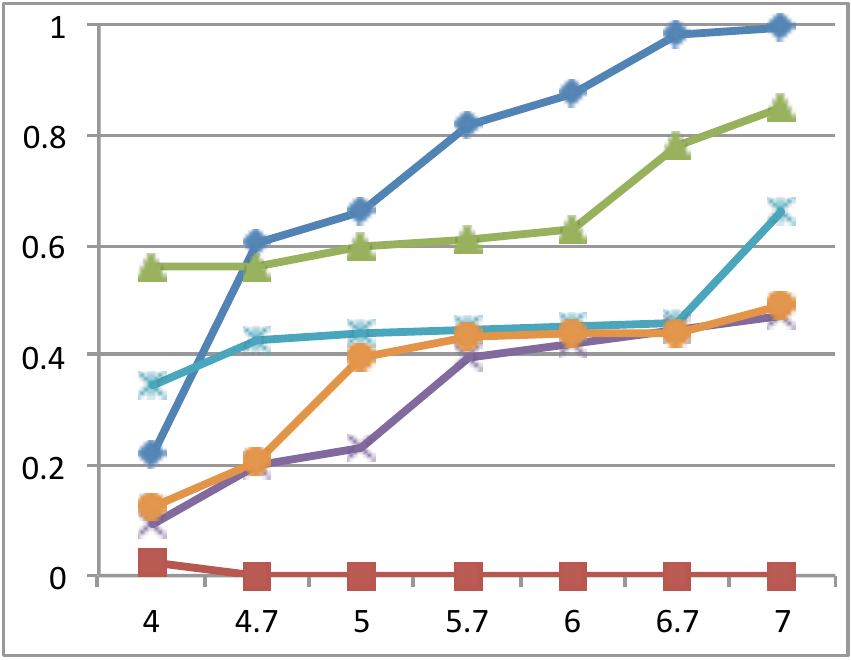}
\includegraphics[width=\myfigurewidth , height=\myfigurehight]{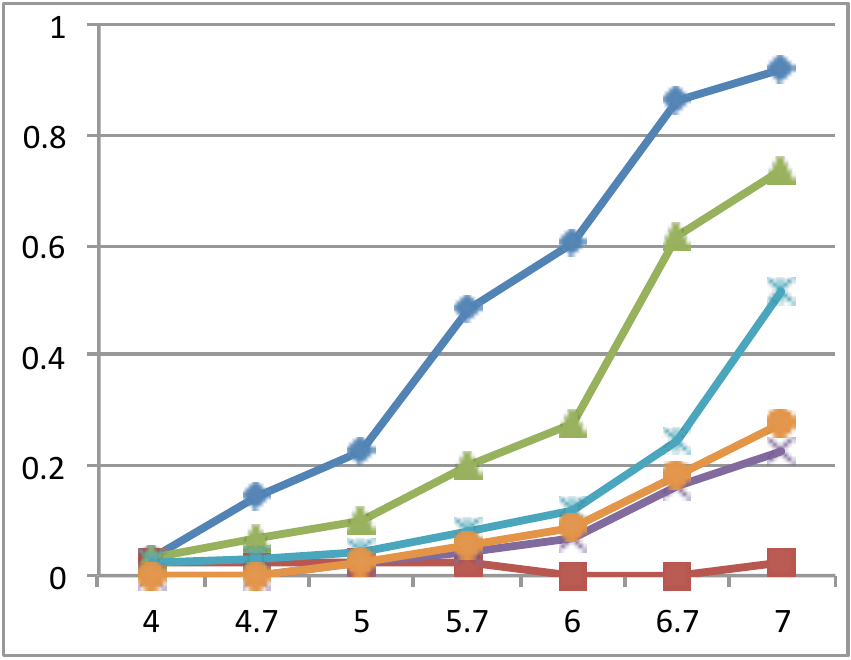}
\includegraphics[width=\myfigurewidth , height=\myfigurehight]{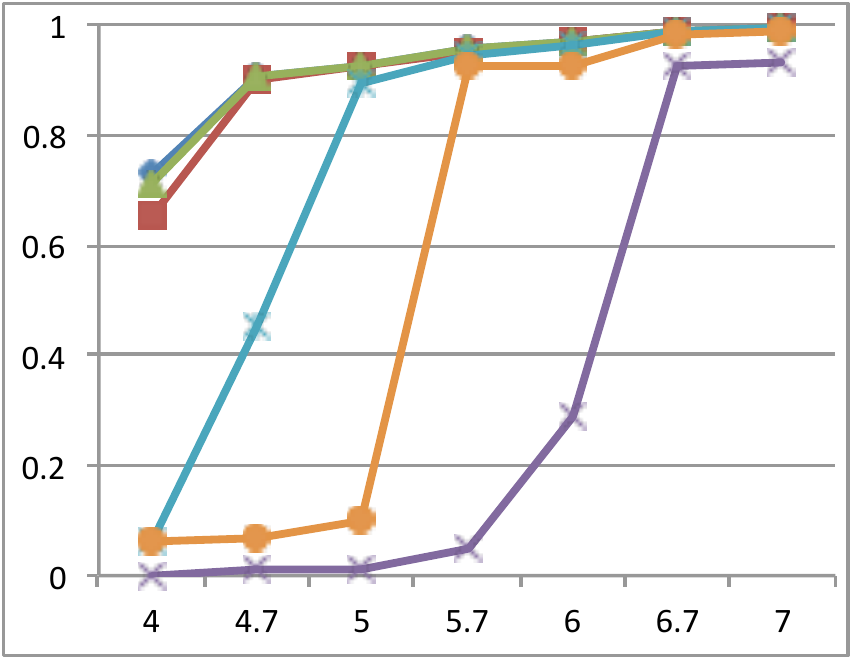}
\includegraphics[width=\myfigurewidth , height=\myfigurehight]{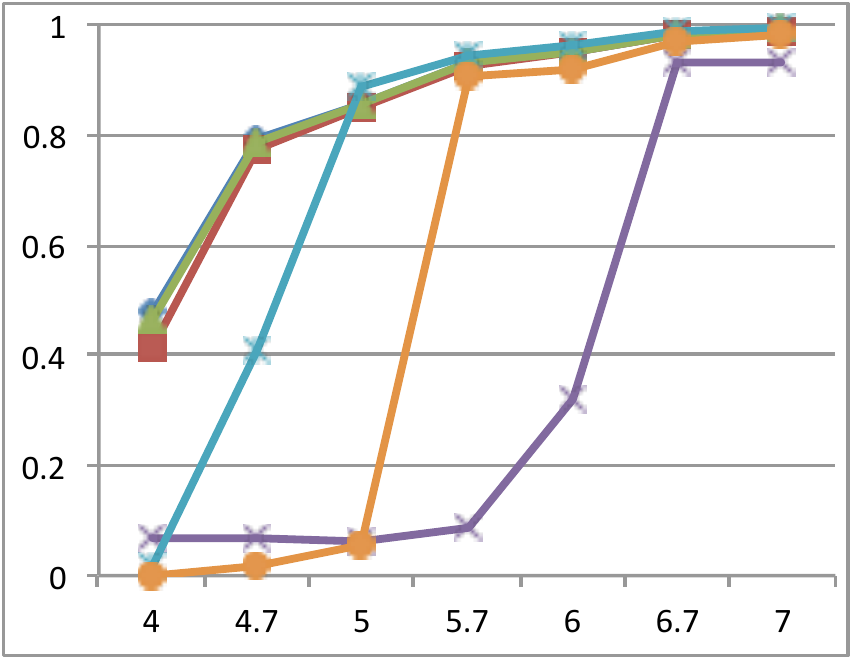}
\includegraphics[width=\myfigurewidth , height=\myfigurehight]{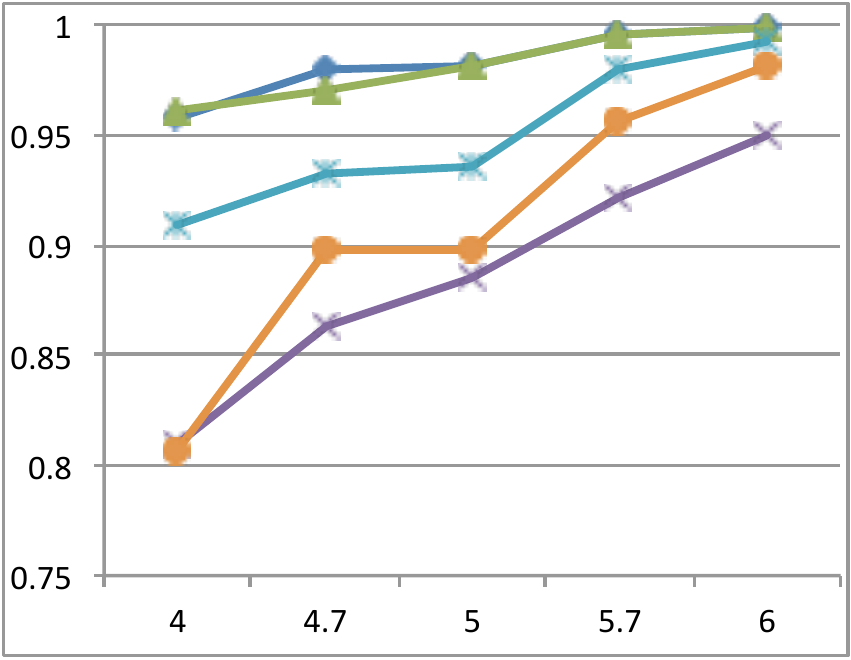}
\includegraphics[width=\myfigurewidth , height=\myfigurehight]{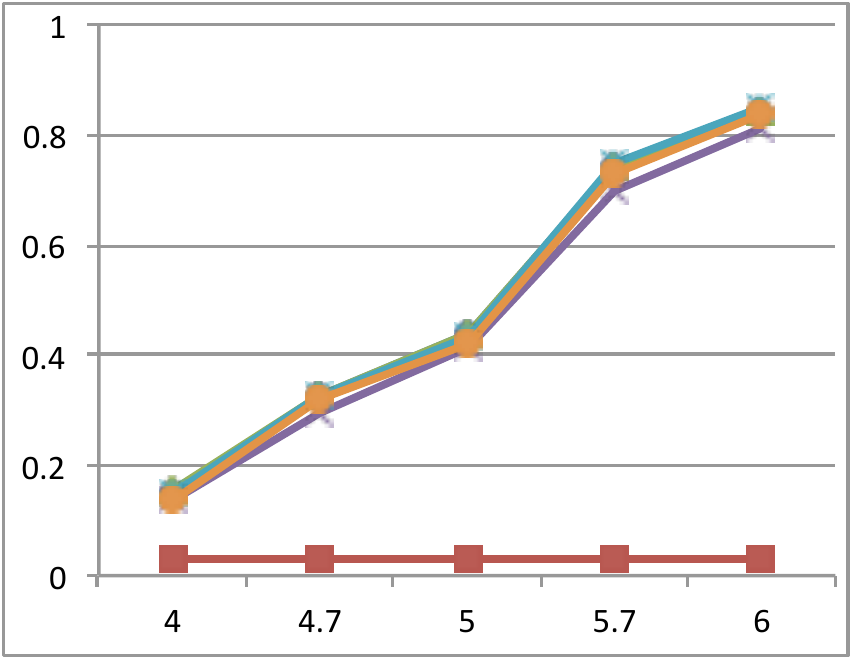}
\includegraphics[width=\myfigurewidth , height=\myfigurehight]{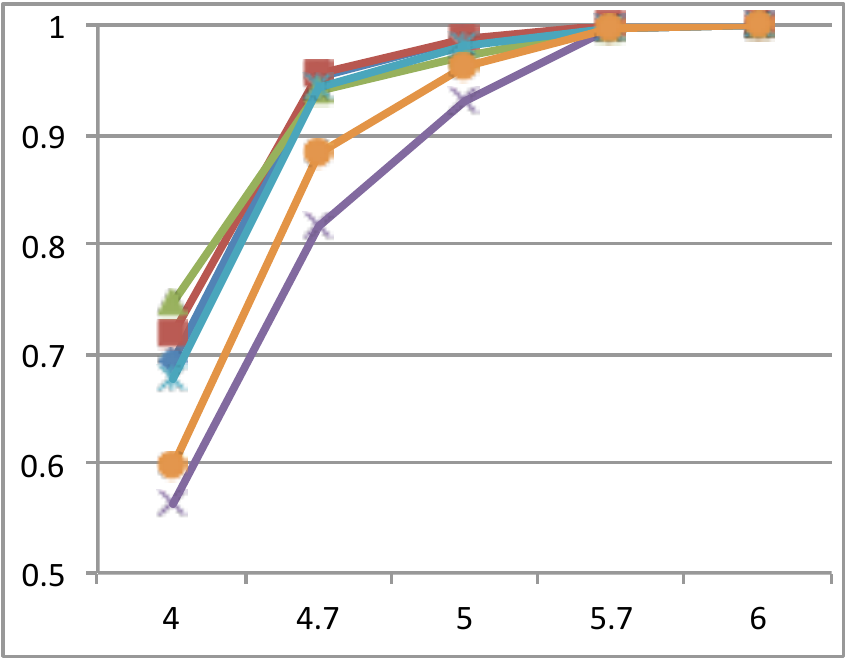}
\includegraphics[width=\myfigurewidth , height=\myfigurehight]{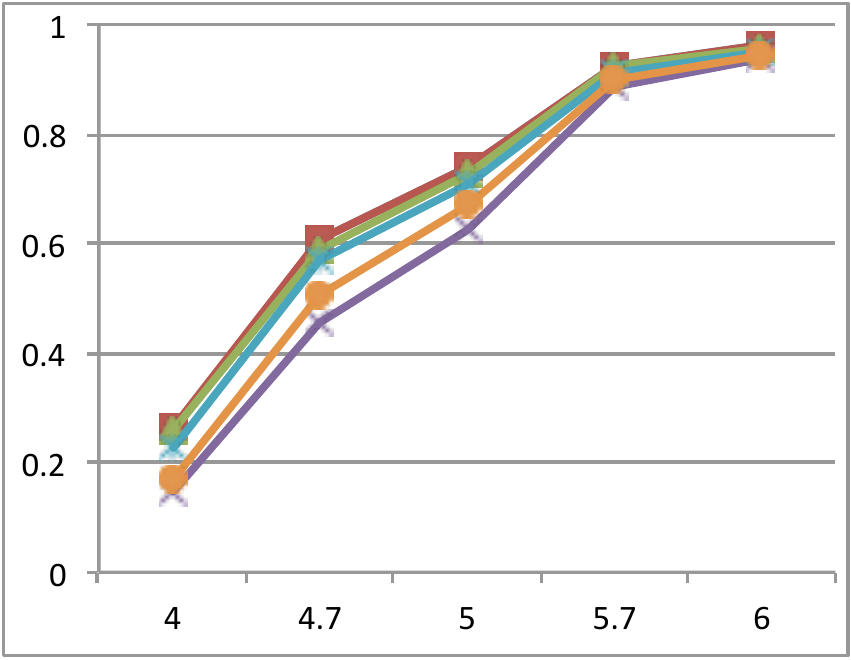}
\caption{
Each horizontal pair of plots corresponds to one matrix. Top to bottom: Wikipedia, Images, Enron, Synthetic . Each left plot shows the quality of approximation ratio, $\|P^k_B A\|_F/\|A_k\|$ (right plots show $\|A Q^k_B\|_F/\|A_k\|$). The number of samples $s$ is on the $x$-axis in log scale $x = \log_{10}(s)$. 
}
\end{center}
\label{plotsplots}
\end{figure}

\subsection{Insights}
The experiments demonstrate three main insights. First and most important, Bernstein-sampling is never worse than any of the other techniques and is often strictly better.  A dramatic example of this is the Wikipedia matrix for which it is far superior to all other methods. The second insight is that L1-sampling, i.e., simply taking $\dist_{ij} = |A_{ij}|/\|A\|_1$, performs rather well in many cases. Hence, if it is impossible to perform more than one pass over the matrix and one can not even obtain an estimate of the ratios of the L1-weights of the rows, L1-sampling seems to be a highly viable option.
The third insight is that for L2-sampling, discarding small entries may drastically improve the performance. However, it is not clear which threshold should be chosen in advance. In any case, in all of the example matrices, both L1-sampling and Bernstein-sampling proved to outperform or perform equally to L2-sampling, even with the correct trimming threshold.

\bibliography{matrixSampling}

\bibliographystyle{alpha}

\newpage

\appendix

\section{Efficient Parallel Reservoir Sampling} \label{sec:rsrvr_sample}
Assume we receive a stream of items each having weight $w_i$.
Further assume that we want to sample a single item from the stream with probability $p_i = w_i/W$ where $W = \sum_i{w_i}$. 
Reservoir sampling is the classic solution to this problem: 
select the very first item in the stream as the ``current" sample and from then on have each successive item $i$ replace the current sample with probability $w_i/W_i$, where $W_i = \sum_{j \le i} w_j$.

Assume now that, instead, we wanted to take $s>1$ items from the stream, but as if the stream was a set and we could sample it \emph{with} replacement. 
One way to do this is to execute $s$ independent reservoir samplers as above in parallel, as was pointed out in~\cite{DrineasZ11}. 
This, however, requires $O(s)$ active memory and $O(s)$ randomized operations \emph{per item in the stream}.

In the formation of the sketch matrix $B$ a potentially large number of samples $s=\nnz(B)$ can make this approach impractical. 
Below we describe an algorithm that requires only $O(\log s)$ \emph{active} memory and $O(1)$ operations per item, instead of $O(s)$ memory and $O(s)$ operations per item, respectively. 
The first idea is to use the fact that samplers are independent.
We can therefore simulate the process above by determining for each item, $a$, the (random) number of samplers, $k$, that would have replaced their current sample with $a$ when it appeared. 
This random variable is Bernouli distributed and can be sampled efficiently. 
If this number is greater than zero, we write item $a$ along with $k$ to durable storage (disk) and process the next item in the stream. 
This processing generates a sketch of the stream on disk, the length of which can  be shown to be bounded by $O(s \log(bN))$, where $b \eqdef \max_i w_i/\min_{i} w_i$. 
Here we can safely assume $w_i > 0$.

When the stream terminates, we process the sketch from \emph{end to beginning} as follows: for each pair $(a,k)$ we encounter in the sketch we process the $k$ update operations as the throwing of $k$ balls into $s$ bins uniformly at random. 
This is because, whether item $a$ replaces the current sample, $a'$, of a particular sampler is independent of $a'$. 
Notice that since we are going over the sketch backwards, the very first ball we place in a bin corresponds to the very last update of the sampler in the original execution. Thus, for each bin, we ignore all but the first ball placement and we stop as soon as each bin has received a ball (thus we also avoid simulating the ``irrelevant" part of the naive computation). 
Performing this simulation only requires a bit-vector of length $s$ in active memory.

Finally, we can avoid even the cost of the bit-vector, as follows. Note that we do not care about the order of the samplers. Only the \emph{number} of samplers that pick any item is important.
Therefore, we can simply track the number of empty bins $\ell$ (samplers that are not committed yet) instead of the whole list and update it every time some balls fall into empty bins.
\begin{algorithm}[h!]
\begin{algorithmic}
\label{sample_alg}
\State {\bf Input:} An integer $s$ and a stream $(a_1,w_1), (a_2,w_2),...$ 
\State $W \leftarrow 0$, \; \; $T \leftarrow$ empty stack
\For{$(a, w) \in$ the stream}
	\State $W \leftarrow W + w$
	\State $p = w/W$
	\State  $k = \operatorname{binomial}(s,p)$ \Comment{Number of reservoir samplers that would have picked item $a$.}
	\If {$k > 0$}
		\State Push $(a,k)$ onto $T$
	\EndIf
\EndFor
\State $\ell = s$ \Comment{$\ell$ holds the number of samplers that did not commit on an item yet.}
\While{$\ell > 0$}
	\State $(a,k) = \operatorname{pop}(T)$
	\State $t = \operatorname{hypergeometric}(s,\ell,k)$
	\If {$t > 0$} \Comment{$t$ samplers committed to item $a$.}
		\State $\ell = \ell - t$
		\State {\bf yield:} $(a,t)$
	\EndIf
\EndWhile
\end{algorithmic}
\end{algorithm}
The hypergeometric distribution $\operatorname{hypergeometric}(s,\ell,k)$ (see e.g \cite{berkopec2007hyperquick} for a more thorough overview) assigns each integer $t$ probability ${\ell \choose t} {s - \ell \choose k - t} / {s \choose k}$. 
In words, assume we have $s$ bins only $\ell$ of which are empty. If we throw $k$ balls to $k$ different bins uniformly at random, the number of balls that fall in empty bins distributes as $\operatorname{hypergeometric}(s,\ell,k)$. 

\end{document}